\documentclass[letterpaper]{article}
\usepackage{aaai2026}
\usepackage{times}
\usepackage{helvet}
\usepackage{courier}
\usepackage[hyphens]{url}
\usepackage{graphicx}
\usepackage{natbib}
\usepackage{caption}
\usepackage[utf8]{inputenc} 
\usepackage[T1]{fontenc}    
\usepackage{booktabs}       
\usepackage{amsfonts}       
\usepackage{nicefrac}       
\usepackage{microtype}      

\usepackage{amsmath}
\usepackage{amssymb}
\usepackage{mathtools}
\usepackage{amsthm}

\usepackage[capitalize,noabbrev]{cleveref}

\usepackage{dsfont}
\usepackage[mathscr]{euscript}
\usepackage{xcolor}
\usepackage{colortbl}
\usepackage{thmtools}
\usepackage{thm-restate}

\usepackage{multirow}

\usepackage{tikz}
\usetikzlibrary{arrows.meta,calc}

\usepackage{pifont}

\usepackage{MnSymbol}
\DeclareMathAlphabet\mathbb{U}{msb}{m}{n}
\usepackage{xpatch}

\def\Rset{\mathbb{R}}

\def\Zset{\mathbb{Z}}

\let\Pr\undefined

\DeclareMathOperator*{\Pr}{\mathbb{P}}

\DeclareMathOperator*{\E}{\mathbb E}
\DeclareMathOperator*{\argmax}{argmax}

\DeclareMathOperator{\sign}{sign}

\DeclarePairedDelimiter{\abs}{\lvert}{\rvert} 
\DeclarePairedDelimiter{\bracket}{[}{]}
\DeclarePairedDelimiter{\curl}{\{}{\}}
\DeclarePairedDelimiter{\paren}{(}{)}

\newcommand{\sC}{{\mathscr C}}
\newcommand{\sD}{{\mathscr D}}
\newcommand{\sE}{{\mathscr E}}

\newcommand{\sH}{{\mathscr H}}

\newcommand{\sM}{{\mathscr M}}

\newcommand{\sX}{{\mathscr X}}
\newcommand{\sY}{{\mathscr Y}}

\newcommand{\hh}{{\sf h}}

\newcommand{\ignore}[1]{}
\newcommand{\pred}{{\sY}_{p}}

\usepackage[toc, page, header]{appendix}
\declaretheorem{theorem}

\newtheorem{definition}[theorem]{Definition}

\title{Beyond Tsybakov: Model Margin Noise and $\sH$-Consistency Bounds}
\author{
Mehryar Mohri\textsuperscript{\rm 1,\rm 2},
Yutao Zhong\textsuperscript{\rm 1}
}
\affiliations {
\textsuperscript{\rm 1} Google Research, New York \\
\textsuperscript{\rm 2} Courant Institute of Mathematical Sciences, New York University \\
mohri@google.com, 
yutaozhong@google.com
}

\begin{document}

\maketitle

\begin{abstract}
  We introduce a new low-noise condition for classification, the
  \emph{Model Margin Noise (MM noise)} assumption, and derive enhanced
  $\sH$-consistency bounds under this condition. MM noise is
  \emph{weaker} than Tsybakov noise condition: it is implied by
  Tsybakov noise condition but can hold even when Tsybakov fails,
  because it depends on the discrepancy between a given hypothesis and
  the Bayes-classifier rather than on the intrinsic distributional
  minimal margin (see Figure~\ref{fig:mm-weaker-than-tsy} for an
  illustration of an explicit example). This hypothesis-dependent
  assumption yields enhanced $\sH$-consistency bounds for both binary
  and multi-class classification. Our results extend the enhanced
  $\sH$-consistency bounds of \citet{mao2025enhanced} with the same
  favorable exponents but under a weaker assumption than the Tsybakov
  noise condition; they interpolate smoothly between linear and
  square-root regimes for intermediate noise levels. We also
  instantiate these bounds for common surrogate loss families and
  provide illustrative tables.
\end{abstract}

\section{Introduction}

The design and analysis of surrogate losses are fundamental to
classification.  Classical analyses established
\emph{Bayes-consistency} for large families of convex surrogates and
derived bounds on the surrogate-to-target excess error in the binary
setting \citep{Zhang2003, bartlett2006convexity,
  steinwart2007compare}, with sharper constants for specific losses
such as $q$-norm SVMs and proper losses \citep{chen2004support,
  reid2009surrogate}.  In multi-class classification, subsequent work
identified which surrogates are Bayes-consistent and which fail (e.g.,
certain hinge variants), while establishing consistency for
sum-exponential/logistic and constrained families
\citep{zhang2004statistical, tewari2007consistency,
  crammer2001algorithmic, WestonWatkins1999, lee2004multicategory}.  A
complementary thread investigated how \emph{growth rates} of these
bounds behave near zero: smoothing often degrades rates, polyhedral
losses can achieve linear behavior, and broad smooth/proper losses
admit square-root lower bounds \citep{mahdavi2014binary,
  frongillo2021surrogate, bao2023proper}.  Yet, these guarantees apply
only to the family of all measurable functions $\sH_{\rm all}$ and are
thus not relevant to the hypothesis set that is actually used in
practice.

To close that gap, \emph{$\sH$-consistency bounds} provide
non-asymptotic guarantees specific to a fixed hypothesis set $\sH$
\citep{awasthi2022h,zhong2025fundamental}, and have since been developed broadly, from
multi-class (including max, sum, constrained, and comp-sum families)
to ranking, structured prediction, abstention/defer, multi-label
learning, adversarial settings, and beyond \citep{awasthi2022multi,
  zheng2023revisiting, mao2023cross, MaoMohriZhong2023ranking,
  MaoMohriZhong2023structured, MaoMohriZhong2024deferral,MaoMohriMohriZhong2023twostage,AwasthiMaoMohriZhong2023theoretically,mao2024multi,cortes2024cardinality,MaoMohriZhong2024}.  More
recently, \citet{mao2025enhanced} showed how \emph{enhanced
  $\sH$-consistency bounds} can be derived by relating
\emph{conditional regrets} via more general inequalities with
non-constant factors depending on the input or predictor
instance. They showed, in particular, that when a lower bound of the
surrogate loss \emph{conditional regret} is given as a power function
of the target \emph{conditional regret} with exponent $s$, this yields
enhanced bounds with low-noise exponents of the form
$\paren*{1 / \paren*{ s - \alpha \paren*{ s - 1 } }}$ under Tsybakov
noise condition.

The Tsybakov noise condition constrains the minimal margin
$\gamma(x) = \Pr(y_{\max} \mid x) - \sup_{y \neq y_{\max}} \Pr(y \mid
x)$ with $y_{\max} = \argmax_{y \in \sY} \Pr(y \mid x)$ and is
therefore purely distributional, describing a worst-case property of
the data, regardless of the hypothesis set $\sH$ being used.  We
introduce a hypothesis-dependent low-noise assumption, the
\emph{Model Margin (MM) noise}, that depends on the \emph{model
  margin}
$\mu(h, x) = \Pr(\hh^*(x) \mid x) - \Pr(\hh(x) \mid x) \geq 0$, where
$h^* \in \sH_{\rm{all}}$ is a Bayes classifier and
$\hh(x) = \argmax_{y \in \sY} h(x, y)$ is the prediction made by
hypothesis $h \colon \sX \times \sY \to \Rset$. This is the gap
between the Bayes label’s conditional probability and the hypothesis's
predicted label’s conditional probability at $x$ This
hypothesis-dependent condition is \emph{weaker} than the
distributional Tsybakov noise condition: since
$\mu\paren*{h,x} \geq \gamma\paren*{x}$ holds whenever
$\hh^*(x) \neq \hh(x)$, any bound on the minimal margin tail
immediately bounds the model margin tail (hence, Tsybakov implies MM),
but the converse need not hold, and it enables $\sH$-consistency
bounds with the same favorable noise exponents, but under a weaker
assumption.

\noindent \textbf{Our contributions.}
We summarize our main results below.
\begin{itemize}

\item \textbf{Model Margin (MM) noise condition.} We introduce the MM
  noise condition (Section~\ref{sec:def}), a hypothesis-dependent
  alternative to Tsybakov noise condition. We formally establish that
  it is a \emph{weaker} condition (``Tsybakov $\Rightarrow$ MM'') and
  can hold even when Tsybakov noise condition fails
  (Theorem~\ref{prop:mm-not-tsy}).

\item \textbf{Key property.} In Section~\ref{sec:property}, we
  establish a key property of the MM noise condition
  (Lemma~\ref{lemma:MM}), which provides an inequality bounding the
  disagreement mass $\E \bracket*{1_{\mu(h, X)> 0}}$ by a power of the
  $0$-$1$ excess error. This property is central to the derivation of
  our $\sH$-consistency bounds.

\item \textbf{Enhanced $\sH$-consistency bounds under MM noise.}  We
  derive enhanced $\sH$-consistency bounds under MM noise for both
  binary (Section~\ref{sec:binary}) and multi-class classification
  (Section~\ref{sec:multi}). These bounds preserve the favorable
  exponents from \citep{mao2025enhanced} while requiring only weaker
  assumptions.

\item \textbf{Applications.} In Section~\ref{sec:more}, we provide
  structural properties (monotonicity and invariance) of MM noise and
  illustrate our bounds for common surrogate losses, including binary
  margin-based losses and multi-class comp-sum losses.

\end{itemize}

\noindent \textbf{Relation to prior work.}  Our results improve
enhanced $\sH$-consistency bounds from \citep{mao2025enhanced} in two
ways.  First, shifting from distributional minimal margin $\gamma$ to
hypothesis-dependent model margin $\mu$ yields a \emph{weaker}
low-noise assumption: the Tsybakov condition implies MM (for any fixed
$\sH$), but MM may still hold when the distribution violates Tsybakov,
especially for restricted hypothesis sets, as $\sH$ may not contain
the 'bad' classifiers that Tsybakov's worst-case margin $\gamma(x)$ is
designed to guard against. Second, this shift produces
\emph{predictor-dependent} constants, quantified by
$\E \bracket*{1_{\mu(h, X) > 0}}^{1 / t}$, which are absent from
Tsybakov-only analyses and which align better with model selection in
practice.  At the same time, we retain the desirable exponents proven
under Tsybakov noise — so practitioners gain bounds with the same
exponents under a weaker assumption.

\section{Related work.}
\label{sec:related}

\textbf{Bayes‐consistency and surrogate losses.}
The study of Bayes‐consistency for convex surrogate losses has a long history.
For binary classification, early foundational analyses by \citet{Zhang2003}, \citet{bartlett2006convexity}, and \citet{steinwart2007compare} established Bayes‐consistency of several convex and margin‐based losses, while also deriving excess‐error or surrogate‐regret bounds. 
Specific examples include $q$‐norm SVM losses with optimal square‐root rates \citep{chen2004support} and tight regret bounds for proper losses \citep{reid2009surrogate}. 

In multi-class classification, analogous results were developed by \citet{zhang2004statistical} and \citet{tewari2007consistency}, who analyzed \emph{max}-, \emph{sum}-, and \emph{constrained}-type surrogates \citep{crammer2001algorithmic,WestonWatkins1999,lee2004multicategory}. 
They demonstrated that max and sum multi-class hinge variants fail to be Bayes‐consistent, whereas sum-exponential and sum-logistic losses, and
constrained families achieve Bayes-consistency. 
Later work unified binary and multi-class analyses under a general supervised learning framework \citep{steinwart2007compare}. 

\textbf{Growth rates and smoothness effects.}
The growth rates of excess-error bounds, that is, the behavior of the bound's functional form $\Gamma$ near zero, have been studied extensively.
Smoothing a hinge‐type loss can worsen this growth \citep{mahdavi2014binary}, while local strong convexity and Lipschitz gradients imply at best square‐root rates \citep{frongillo2021surrogate,bao2023proper}. 
Polyhedral losses attain linear rates \citep{finocchiaro2019embedding}, clarifying why piecewise‐linear surrogates such as the hinge are statistically optimal in that sense. 
These results, however, apply only to the family of all measurable functions and thus ignore the hypothesis‐set choices that dominate practical performance.

\textbf{Hypothesis‐dependent analysis.}
Bayes‐consistency guarantees are asymptotic and model‐agnostic. 
As emphasized by \citet{long2013consistency} and \citet{zhang2020bayes}, a Bayes‐consistent surrogate may yield constant test error on restricted model families, while an inconsistent one may succeed. 
This observation motivated the introduction of \emph{$\sH$‐consistency bounds} by \citet{awasthi2022h}, which relate the target estimation error within a restricted hypothesis set $\sH$ to the surrogate estimation error in a non‐asymptotic manner.

\textbf{$\sH$‐consistency bounds.}
Following the binary framework of \citet{awasthi2022h}, \citet{awasthi2022multi} extended $\sH$‐consistency to multi-class settings, covering \emph{max}, \emph{sum}, and \emph{constrained loss} families \citep{crammer2001algorithmic,weston1998multi,lee2004multicategory}.
\citet{mao2023cross} further extended these analyses to the \emph{comp‐sum losses}, encompassing cross‐entropy, generalized cross‐entropy, mean absolute error, and other hybrid surrogate losses.
A general characterization for comp‐sum and constrained losses was later provided in \citet{MaoMohriZhong2023characterization}.
Recent refinements revealed that smooth surrogates across both binary and multi-class settings exhibit a universal square‐root growth rate \citep{MaoMohriZhong2024}.
The $\sH$‐consistency framework has also been adapted to ranking \citep{MaoMohriZhong2023rankingabs,MaoMohriZhong2023ranking}, abstention and rejection learning \citep{MaoMohriZhong2023predictor,MaoMohriZhong2024score,MohriAndorChoiCollinsMaoZhong2024learning}, learning to defer \citep{MaoMohriZhong2024deferral,mao2024realizable,MaoMohriZhong2025mastering,MaoMohriMohriZhong2023twostage,mao2025theory,desalvo2025budgeted}, top‐$k$ classification \citep{cortes2024cardinality}, multi-label learning \citep{mao2024multi}, adversarial robustness \citep{awasthi2021calibration,awasthi2021finer,awasthi2023dc,AwasthiMaoMohriZhong2023theoretically}, bounded regression \citep{mao2024regression,mao2024h}, optimization of generalized metrics \citep{MaoMohriZhong2025principled}, imbalanced learning \citep{cortes2025balancing,cortes2025improved},  and structured prediction \citep{MaoMohriZhong2023structured}.

\textbf{Enhanced $\sH$‐consistency bounds.}
\citet{mao2025enhanced} generalized the previous setting by introducing \emph{enhanced $\sH$‐consistency bounds} based on refined inequalities between surrogate and target conditional regrets.
This produced distribution‐dependent exponents of the form $1 / \paren*{ s - \alpha \paren*{ s - 1 } }$ under the Tsybakov noise assumption across binary and multi-class classification.
Nevertheless, the Tsybakov noise condition constrains only the minimal margin $\gamma(x)$ and is thus purely distributional.

\textbf{This work: model‐dependent low‐noise conditions.}
Our results strengthen the above framework by replacing the minimal margin $\gamma$ with the \emph{model margin} $\mu$, which depends on both the data distribution and the chosen hypothesis. 
This yields the \emph{Model Margin (MM) noise} assumption, a \emph{weaker}, hypothesis‐dependent condition that is implied by the classical Tsybakov noise assumption but remains valid for a broader range of models and distributions, as it measures noise relative to each hypothesis rather than the minimal margin alone.
Under MM noise, $\sH$‐consistency bounds preserve the same exponent as in the Tsybakov case.
We further establish the theoretical robustness of the MM noise condition by demonstrating its monotonicity with respect to hypothesis class inclusion and its invariance under monotone score transformations.

In summary, MM noise extends enhanced $\sH$‐consistency bounds to a weaker yet more flexible, hypothesis‐dependent setting, allowing more adaptive generalization guarantees with the same favorable exponents to be established for both binary and multi-class surrogate losses.

\section{Preliminaries}
\label{sec:prelim}

\textbf{Learning setup and notation.}
We consider a supervised learning problem with an unknown distribution $\sD$ over pairs $\sX \times \sY$, where $\sX$
is the input space and $\sY$ is the label space. A hypothesis $h$ is selected from a hypothesis set $\sH \subseteq \sH_{\rm all} \coloneqq \curl*{ h \colon \sX \to \pred \mid h \text{ measurable} }$, where $\pred$ denotes the prediction space that specifies the form of model outputs. 
For instance, $\pred = \Rset$ for scalar scores in binary classification, and $\pred = \Rset^n$ for vector‐valued scores in multi‐class classification, where $n \in \Zset_{+}$ is the number of labels.

A loss function $\ell \colon \sH \times \sX \times \sY \to \Rset_{+}$ measures the prediction error. 
Its generalization error and the best‐in‐class generalization error within $\sH$ are defined as
\begin{equation*}
\sE_{\ell}(h) \coloneqq \E_{(x, y) \sim \sD} \bracket*{\ell(h, x, y)},
\qquad
\sE_{\ell}^*(\sH) \coloneqq \inf_{h \in \sH} \sE_{\ell}(h).
\end{equation*}

\noindent \textbf{Conditional errors.}  For every input $x \in \sX$,
we define the \emph{conditional error} and \emph{best-in-class
  conditional error} as
\begin{equation*}
\sC_{\ell}\paren*{h, x} \coloneqq \E_{y \mid x} \bracket*{ \ell\paren*{h, x, y} }, \quad \sC^*_{\ell}\paren*{\sH, x} = \inf_{h \in \sH} \sC_{\ell}\paren*{h, x}.
\end{equation*}
The generalization error can be rewritten as
$\sE_{\ell}(h) = \E_{X} \bracket*{ \sC_{\ell}\paren*{h, x} }$.  We
also define the \emph{conditional regret}, (how suboptimal $h$ is at a
single point $x$) and \emph{estimation error} as:
\begin{equation*}
  \Delta \sC_{\ell, \sH}\paren*{h, x}
  \coloneqq \sC_{\ell}\paren*{h, x} - \sC^*_{\ell}\paren*{\sH, x},
\quad
\sE_{\ell}(h) - \sE_{\ell}^*\paren*{\sH}.
\end{equation*}
The \emph{minimizability gap} (a technical term capturing how well
$\sH$ can optimize the loss pointwise) is defined as
\begin{equation*}
  \sM_{\ell}\paren*{\sH}
  \coloneqq \sE_{\ell}^*\paren*{\sH} - \E_{X} \bracket*{\sC^*_{\ell}\paren*{\sH, x}}
  \geq 0.
\end{equation*}
When $\sH$ is sufficiently rich (for example, $\sH = \sH_{\rm all}$ or
$\sE^*_{\ell}(\sH) = \sE^*_{\ell}\paren*{\sH_{\rm all}}$), this gap
vanishes. In general, it is non-zero and can be upper bounded by the
approximation error
$\sE_{\ell}^*\paren*{\sH} - \sE_{\ell}^*\paren*{\sH_{\rm all}}$
\citep{MaoMohriZhong2024}.

\noindent \textbf{$\sH$‐consistency bounds.}  Let $\ell_{1}$ be a
surrogate loss and $\ell_{2}$ the target loss.  An
\emph{$\sH$‐consistency bound} relates their estimation errors via a
non-asymptotic bound:
\begin{align*}
& \sE_{\ell_{2}}\paren*{ h } - \sE_{\ell_{2}}^*\paren*{ \sH } + \sM_{\ell_{2}}\paren*{ \sH }\\
& \qquad \leq
  \Gamma\paren*{ \sE_{\ell_{1}}\paren*{ h } - \sE_{\ell_{1}}^*\paren*{ \sH }
  + \sM_{\ell_{1}}\paren*{ \sH } },
\end{align*}
for a non-decreasing concave function $\Gamma$ with
$\Gamma\paren*{ 0 } = 0$ \citep{mao2023cross}.  This guarantee shows
that reducing the surrogate estimation error implies a proportional
reduction in the target error within the same hypothesis set.

\noindent \textbf{Enhanced $\sH$-consistency bound.}
An enhanced form introduces a multiplicative hypothesis-dependent factor $\gamma(h)$:
\begin{align*}
& \sE_{\ell_{2}}\paren*{ h } - \sE_{\ell_{2}}^*\paren*{ \sH } + \sM_{\ell_{2}}\paren*{ \sH }\\
& \qquad \leq
\Gamma\paren*{ \gamma\paren*{ h } \paren*{ \sE_{\ell_{1}}\paren*{ h } - \sE_{\ell_{1}}^*\paren*{ \sH } + \sM_{\ell_{1}}\paren*{ \sH } }}.
\end{align*}
In our analysis, $\Gamma$ will be expressed as the power function
$\Gamma\paren*{ u } = u^{ 1 / s }$ with $s \geq 1$, and
$\gamma\paren*{ h }$ will be expressed in terms of the model margin
via the term $\E_{X}\bracket*{ 1_{ \mu(h, X) > 0 } }^{\frac{1}{t}}$,
where $t$ is the conjugate of $s$ (i.e.,
$\frac{1}{s} + \frac{1}{t} = 1$).

\noindent \textbf{Pointwise bounds.}  An $\sH$‐consistency bound
typically follows from a pointwise bound relating the conditional
regrets:
\begin{equation*}
\Delta \sC_{ \ell_{2}, \sH }\paren*{ h, x }  \leq \Gamma \paren*{ \Delta \sC_{ \ell_{1}, \sH }\paren*{ h, x } }.
\end{equation*}
Taking expectation over $X$ and applying Jensen's inequality yields a
$\sH$‐consistency bound.\ignore{ A common form is the power function
$\Gamma\paren*{ u } = u^{ 1 / s }$ with $s \geq 1$ (and conjugate $t$
such that $1 / s + 1 / t = 1$), which is particularly useful in
practice.}

\noindent \textbf{Binary classification.}  For binary classification,
the label space is $\sY = \curl*{ -1, +1 }$ and the prediction space
is $\pred = \Rset$.  The target loss is the binary zero–one loss:
\begin{equation*}
\ell^{\rm{bi}}_{0-1}\paren*{ h, x, y } = 1_{ \sign\paren*{ h\paren*{ x } } \neq y },
\qquad
\sign\paren*{ t } =
\begin{cases}
+1, & t \geq 0,\\
-1, & t < 0.
\end{cases}
\end{equation*}
Let $\hh(x) = \sign(h(x))$ be the hypothesis's prediction and $\eta(x) = \paren*{Y = +1 \mid X = x}$ be the conditional probability of $Y = +1$ given $X = x$. 
For any loss function $\ell$, the conditional error can be written as
\begin{equation*}
\sC_{\ell}\paren*{ h, x }
= \eta\paren*{ x } \ell\paren*{ h, x, +1 }
+ \paren*{ 1 - \eta\paren*{ x } } \ell\paren*{ h, x, -1 }.
\end{equation*}
Typical surrogates include margin-based losses $\ell_{\Phi}\paren*{ h, x, y } = \Phi\paren*{ y h(x) }$, with $\Phi$ convex, non-increasing, and nonnegative.

\noindent \textbf{Multi-class classification.}
In the multi-class setting, $\sY = [n] = \curl*{ 1, \dots, n }$ and $\pred = \Rset^{ n }$. 
The scaler $h\paren*{ x, y }$ denotes the score assigned to label $y$, and the predicted label is $\hh\paren*{ x } = \argmax_{ y \in \sY } h\paren*{ x, y }$ (with a fixed deterministic tie-breaking rule). 
The target loss is the multi-class zero–one loss
\begin{equation*}
\ell_{0-1}\paren*{ h, x, y } = 1_{ \hh\paren*{ x } \neq y }.
\end{equation*}
Let $\Pr \paren*{ y \mid x }$ be conditional probability of
$y$ given $x$. The conditional error is
\begin{equation*}
\sC_{\ell}\paren*{ h, x } = \sum_{ y \in \sY } \Pr \paren*{ y \mid x } \ell\paren*{ h, x, y }.
\end{equation*}
Common surrogate families include max losses \citep{crammer2001algorithmic}, constrained losses \citep{lee2004multicategory}, and comp–sum losses \citep{mao2023cross}.  The following result from
\citet{awasthi2022multi} characterizes the conditional regret
of the multi-class zero-one loss.
\begin{restatable}{lemma}{ExplicitAssumption}
\label{lemma:explicit_assumption_01}
For every input $x \in \sX$,
\begin{align*}
\sC^*_{ \ell_{0-1}}\paren*{ \sH, x } &= 1 - \max_{ y \in \curl*{\hh(x) \colon h \in \sH}} \Pr \paren*{ y \mid x },\\
\Delta \sC_{ \ell_{0-1}, \sH }\paren*{ h, x }
&= \max_{ y \in \curl*{\hh(x) \colon h \in \sH}} \Pr \paren*{ y \mid x } - \Pr \paren*{ \hh\paren*{ x } \mid x }.
\end{align*}
\end{restatable}

\ignore{
\noindent \textbf{Fundamental tool.}
The following result from \citet{mao2025enhanced} provides a fundamental tool for deriving enhanced $\sH$-consistency bounds.

\begin{restatable}{theorem}{NewBoundPower}
\label{Thm:new-bound-power}
Assume there exist measurable functions $\alpha, \beta$ mapping from $\sH \times \sX$ to the non-negative real numbers such that $\sup_{x}\alpha\paren*{ h, x } < +\infty$ and $\E_{ X }\bracket*{ \beta\paren*{ h, X } } = 1$, and suppose that for some $s \geq 1$ (with conjugate $t$ satisfying $\frac{1}{s} + \frac{1}{t} = 1$),
\begin{equation*}
\frac{ \Delta \sC_{ \ell_2, \sH }\paren*{ h, x } }{ \beta\paren*{ h, x } }
\leq
\paren*{ \alpha\paren*{ h, x }  \Delta \sC_{ \ell_1, \sH }\paren*{ h, x } }^{ \frac{1}{s} }.
\end{equation*}
Then, with $\gamma\paren*{ h } = \E_{ X } \bracket*{ \alpha^{ \frac{t}{s} }\paren*{ h, X } \beta^{ t }\paren*{ h, X } }^{ \frac{1}{t} }$,
\begin{align*}
& \sE_{ \ell_2 }\paren*{ h } - \sE_{ \ell_2 }^*\paren*{ \sH } + \sM_{ \ell_2 }\paren*{ \sH }\\
& \qquad \leq
\gamma\paren*{ h } 
\bracket*{ \sE_{ \ell_1 }\paren*{ h } - \sE_{ \ell_1 }^*\paren*{ \sH } + \sM_{ \ell_1 }\paren*{ \sH } }^{ \frac{1}{s} }.
\end{align*}
\end{restatable}
We will show below that $\gamma\paren*{ h }$ can be expressed in terms of the model margin, providing hypothesis-dependent constants that are tighter than the unit constant in the standard $\sH$-consistency bounds while keeping the same exponent.
}

\section{Model Margin (MM) Noise Assumption}
\label{sec:mm}

This section introduces the Model Margin (MM) noise condition, a
hypothesis-dependent low-noise assumption. In contrast to classical
Tsybakov noise assumption, the proposed condition depends on a
predictor $h$ via its \emph{model margin}.

\subsection{Definition and comparison with Tsybakov noise}
\label{sec:def}

Let $h^* \in \sH_{\rm all}$ be a Bayes classifier. By
\citep[Lemma~2.1]{MaoMohriZhong2024}, there exists indeed a measurable
function $h^*$ and for all $x \in \sX$, it satisfies
$\Pr(\hh^*(x) \mid x) = \max_{y \in \sY} \Pr(y \mid x)$. For each
input $x$, define the \emph{model margin}
\begin{equation*}
\mu\paren*{h, x} \coloneqq \Pr\paren*{ \hh^*(x) \mid x } - \Pr\paren*{ \hh(x) \mid x } \geq 0.
\end{equation*}
Note that $\mu(h, x) = \Delta \sC_{\ell_{0-1}, \sH_{\rm all}}(h, x)$ by Lemma~\ref{lemma:explicit_assumption_01}.

\begin{definition}[\emph{Model Margin (MM) noise}]
There exist $B > 0$ and $\alpha \in \bracket*{0, 1}$ such that for all $h \in \sH$ and all $t > 0$,
\begin{equation*}
\Pr\bracket*{ 0 < \mu\paren*{h, X} \leq t } \leq B t^{ \frac{\alpha}{1 - \alpha} }.
\end{equation*}
\end{definition}

Intuitively, this condition requires that for each hypothesis $h$, the
probability mass of points where $h$ disagrees with the Bayes
classifier but with a small model‐dependent margin is controlled. It
is therefore a \emph{weaker} assumption than Tsybakov noise
\citep{MammenTsybakov1999}.

\begin{definition}[Tsybakov noise \cite{mao2025enhanced}]
There exist $B > 0$ and $\alpha \in \bracket*{0, 1}$ such that
\begin{equation*}
\forall t > 0, \quad \Pr[\gamma(X) \leq t] \leq B t^{\frac{\alpha}{1 - \alpha}}.
\end{equation*}
where
$\gamma(x) = \Pr(y_{\max} \mid x) - \sup_{y \neq y_{\max}} \Pr(y \mid
x)$ with $y_{\max} = \argmax_{y \in \sY} \Pr(y \mid x)$ is the minimal
margin for a point $x \in \sX$.
\end{definition}

The Tsybakov condition is strong because it is independent of any
specific hypothesis. However, this is also its main limitation. It can
be overly pessimistic if the regions where the minimal margin
$\gamma(x)$ is small are regions where the classifiers in $\sH$
already perform well (i.e., they predict the Bayes label
$\hh^*(x)$). Our MM noise condition, $\mu(h, x)$, is designed to
handle exactly this scenario, as $\mu(h, x)$ becomes zero for such
correct predictions, effectively ignoring the small minimal margin.

\paragraph{Tsybakov $\Rightarrow$ MM (MM is weaker).}
Let $\eta_1(x) = \Pr(y_{\max} \mid x) = \Pr(\hh^*(x) \mid x)$ be the
probability of the Bayes-optimal label and
$\eta_2(x) = \max_{y \neq \hh^*(x)} \Pr(y \mid x)$ be the probability
of the most likely incorrect label. The minimal margin is
$\gamma(x) = \eta_1(x) - \eta_2(x)$.  On any $x$ where
$\hh \paren*{x} \neq \hh^*\paren*{x}$, we have
$\Pr(\hh(x) \mid x) \leq \eta_2(x)$. Thus,
\begin{align*}
\mu\paren*{h, x} &= \Pr \paren*{\hh^*(x) \mid x} - \Pr \paren*{\hh(x) \mid x}\\
&\geq \eta_1(x) - \eta_2(x)\\
&= \gamma\paren*{x}.
\end{align*}
Since $\mu(h, x) = 0$ when $\hh(x) = \hh^*(x)$, the region of disagreement is $\curl*{x : \mu(h, x) > 0}$. This implies that $\mu(h, x) > 0$ only if $\hh(x) \neq \hh^*(x)$, in which case $\mu(h, x) \geq \gamma(x)$.
Therefore, for all $t > 0$,
\begin{align*}
\curl*{ 0 < \mu\paren*{h, X} \leq t } &\subseteq \curl*{\gamma\paren*{X} \leq t } \\
\Pr\bracket*{ 0 < \mu\paren*{h, X} \leq t } &\leq \Pr\bracket*{\gamma \paren*{X} \leq t }.
\end{align*}
Any Tsybakov noise tail bound on $\gamma$ immediately yields the MM noise tail bound for all $h\in\sH$.

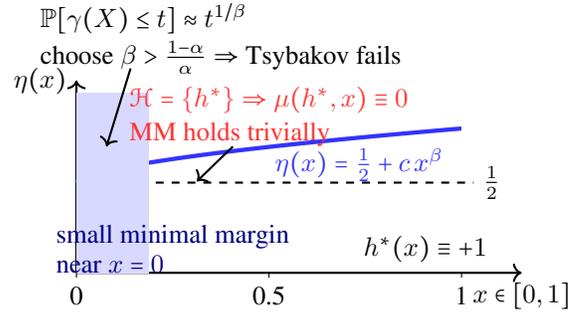
\begin{figure}[t]
\centering
\begin{tikzpicture}[scale=.8]
  \draw[->, thick] (0,0) -- (7.4,0) node[below] {$x \in [0, 1]$};
  \draw[->, thick] (0,0) -- (0,3.2) node[left] {$\eta(x)$};

  \draw[dashed, thick] (0,1.5) -- (6.6,1.5);
  \node at (6.9,1.5) {$\tfrac{1}{2}$};

  \draw[domain=0:6.4, smooth, samples=150, blue!80, ultra thick]
    plot(\x, {1.5 + 0.9*pow(\x/6.4,0.6)});
  \node[blue!80] at (4.7,1.8) {$\eta(x) = \tfrac{1}{2} + c\,x^{\beta}$};

  \fill[blue!15] (0,0) rectangle (1.2,3.0);
  \node[blue!50!black, align=left] at (1.6,0.4)
    {small minimal margin\\ near $x = 0$};

  \node[black] at (5.8,0.4) {$h^*(x) \equiv +1$};

  \draw[->, thick] (2.6,2.1) -- (2.0,1.52);
  \node[red!80, align=left] at (3.2,2.6)
    {$\sH = \curl*{h^*}  \Rightarrow \mu(h^*, x)\equiv 0$\\[2pt]
     MM holds trivially};

  \draw[->, thick] (0.9,3.4) -- (0.5,2.1);
  \node[black, align=left] at (2.4,3.9)
    {$\Pr\bracket*{\gamma(X) \leq t} \approx t^{1/\beta}$\\[2pt]
     choose $\beta > \tfrac{1 - \alpha}{\alpha}$
     $\Rightarrow$ Tsybakov fails};

  \foreach \x/\lab in {0/0, 3.2/0.5, 6.4/1}
    \draw (\x,0) -- (\x,-0.08) node[below] {\lab};
  \end{tikzpicture}
\caption{ \textbf{MM holds while Tsybakov fails.}  Consider
  $X \sim \mathrm{Unif}[0, 1]$ and $\eta(x) = \tfrac{1}{2} + c\, x^{\beta}$
  (with $c \in \paren*{0, \tfrac12}, \beta > 0$).  The Bayes classifier is
  $h^*(x) \equiv +1$ and the minimal margin is
  $\gamma(x) = 2c\, x^{\beta}$.  The Tsybakov tail is
  $\Pr \bracket*{\gamma(X) \leq t} = \paren*{ \tfrac{t}{2c}
  }^{1/\beta}$. If we choose $\beta > \tfrac{1 - \alpha}{\alpha}$, this
  tail is "heavy" and the Tsybakov condition fails for $\alpha$.
  However, if we use the restricted class $\sH = \curl*{h^*}$, the model
  margin is $\mu(h^*, x)\equiv 0$. Thus,
  $\Pr \bracket*{0 < \mu(h^*, X) \leq t} = 0$, and the MM noise condition
  holds trivially. }
\label{fig:mm-weaker-than-tsy}
\end{figure}

\begin{theorem}[MM $\not\Rightarrow$ Tsybakov]
\label{prop:mm-not-tsy}
There exist a distribution $\sD$ and a hypothesis set $\sH$ such that the MM noise condition holds while the Tsybakov noise condition fails.

\end{theorem}

\begin{proof}
\textbf{Explicit example (see Figure~\ref{fig:mm-weaker-than-tsy} for illustration).}
Let $\sX = [0,1]$ with $X \sim \mathrm{Unif}[0, 1]$, $\sY = \curl*{-1, +1}$, and define
\begin{equation*}
\eta(x) \coloneqq \Pr \paren*{Y = +1 \mid X = x} = \tfrac{1}{2} + c x^{\beta}, \quad c \in \paren*{0, \tfrac{1}{2}}, \beta > 0.
\end{equation*}
Then one Bayes classifier is $h^*(x) \equiv +1$ and the minimal margin is $\gamma(x) = \abs*{2\eta(x)-1} = 2c\, x^{\beta}$.
For any $t>0$,
\begin{equation*}
\Pr\bracket*{\gamma(X) \leq t} = \Pr\paren*{ X \leq \paren*{ \tfrac{t}{2c} }^{1/\beta} } = \paren*{ \tfrac{t}{2c} }^{1/\beta}.
\end{equation*}
Fix a target Tsybakov exponent $\alpha \in [0,1)$ and choose $\beta > \tfrac{1-\alpha}{\alpha}$ (e.g., for $\alpha = \tfrac{1}{2}$ take $\beta > 1$). Then $t^{1/\beta}$ decays more slowly than $t^{\alpha/(1-\alpha)}$, so the Tsybakov noise condition fails:
\begin{align*}
& \Pr\bracket*{ \gamma(X) \leq t } \not\leq B\, t^{\alpha/(1-\alpha)}\\
& \quad \text{for any } B>0 \text{ and sufficiently small } t.
\end{align*}

Now let $\sH \coloneqq \curl*{ h^* }$ be the singleton hypothesis set. For $h = h^*$, $\mu \paren*{h, x} \equiv 0$, hence
\begin{equation*}
\Pr \bracket*{ 0 < \mu \paren*{h^*, X} \leq t } = 0 \quad \text{for all } t > 0,
\end{equation*}
which trivially satisfies the MM noise tail bound for any $\alpha$ and $B$. Therefore, MM noise condition holds while Tsybakov noise condition fails for the same $\alpha$ and $B$.    
\end{proof}

\subsection{Key property}
\label{sec:property}

Next, we establish a key property of the MM noise condition, which provides an inequality bounding the disagreement mass $\E \bracket*{1_{\mu(h, X)> 0}}$ by a power of the $0$–$1$ excess error.
\begin{restatable}[Disagreement mass vs.\ $0$–$1$ excess error]{lemma}{MM}
\label{lemma:MM}
Under MM noise, there exists $c = \tfrac{ B^{ 1-\alpha } }{ \alpha^{\alpha} } > 0$ such that for all $h \in \sH$,
\begin{align*}
\E\bracket*{ 1_{ \mu\paren*{h, X} > 0 } } 
& \leq c \E[\mu(h, X) 1_{\mu(h, X) > 0}]^\alpha \\
&= c \paren*{ \sE_{\ell_{0-1}}\paren*{h} - \sE_{\ell_{0-1}}\paren*{h^*} }^{\alpha}.
\end{align*}
\end{restatable}
\begin{proof}
By definition of the expectation and the Lebesgue integral, for any $u > 0$,
\begin{align*}
& \E\bracket*{ \mu(h, X) 1_{\mu(h, X) > 0} }\\
&= \int_{0}^{+\infty} \Pr\bracket*{ \mu(h, X) 1_{\mu(h, X) > 0} > t } \, dt\\
&\geq \int_{0}^{u} \Pr\bracket*{ \mu(h, X) > t } \, dt.  
\end{align*}
Since the equality $\Pr\bracket*{ \mu(h, X) > t } = \Pr\bracket*{ \mu(h, X) > 0 } - \Pr\bracket*{ 0 < \mu(h, X) \leq t }$ holds, and by the MM noise assumption $\Pr\bracket*{ 0 < \mu(h, X) \leq t } \leq B t^{\frac{\alpha}{1 - \alpha}}$, we have
\begin{align*}
& \E\bracket*{ \mu(h, X) 1_{\mu(h, X) > 0} }\\
& \geq \int_{0}^{u} \paren*{ \Pr\bracket*{ \mu(h, X) > 0 } - B t^{\frac{\alpha}{1 - \alpha}} } \, dt\\
& = u \Pr\bracket*{ \mu(h, X) > 0 } - B (1 - \alpha) u^{\frac{1}{1 - \alpha}}.    
\end{align*}
Maximizing the right-hand side over $u > 0$ gives
\begin{align*}
u^{*} &= \paren*{ \frac{ \Pr\bracket*{ \mu(h, X) > 0 } }{ B } }^{\frac{1 - \alpha}{\alpha}},\\
\E\bracket*{ \mu(h, X) 1_{\mu(h, X) > 0} }
&\geq \alpha B^{-\frac{1 - \alpha}{\alpha}} \Pr\bracket*{ \mu(h, X) > 0 }^{\frac{1}{\alpha}}.    
\end{align*}
Rearranging,
$$
\Pr\bracket*{ \mu(h, X) > 0 }
\leq \frac{ B^{1 - \alpha} }{ \alpha^{\alpha} } \paren*{ \E\bracket*{ \mu(h, X) 1_{\mu(h, X) > 0} } }^{\alpha}.
$$
By Lemma~\ref{lemma:explicit_assumption_01}, 
\begin{align*}
\E\bracket*{ \mu(h, X) 1_{\mu(h, X) > 0} } 
& = \E \bracket*{\Delta \sC_{\ell_{0-1}, \sH_{\rm all}}(h, x)}\\
& = \sE_{\ell_{0-1}}(h) - \sE_{\ell_{0-1}}(h^*).    
\end{align*}
Thus the claim holds with $c = \frac{ B^{1 - \alpha} }{ \alpha^{\alpha} }$.
\end{proof}

The left-hand side of Lemma~\ref{lemma:MM} quantifies the disagreement mass between the hypothesis $h$ and a Bayes classifier. The right-hand side establishes that this probability is bounded by a power of the $0$–$1$ excess error. This property is central to the derivation of our $\sH$-consistency bounds.

\section{$\sH$-Consistency Bounds under MM Noise}
\label{sec:mm-bounds}

We now derive the main $\sH$-consistency bounds under the MM noise assumption, first for the multi-class setting and then for the binary case.

\subsection{Multi-class classification}
\label{sec:multi}

The following result provides the $\sH$-consistency bound based on the notion of $\mu(h, x)$. 

\begin{restatable}{theorem}{MMMulti}
\label{Thm:mm-multi}
Suppose there exists $s \geq 1$  with
conjugate number $t \geq 1$, that is $\frac{1}{s} + \frac{1}{t} = 1$, such that
\begin{equation*}
\forall (h, x) \in \sH \times \sX, \quad \Delta \sC_{ \ell_{0-1}, \sH }\paren*{ h, x } \leq \paren*{ \Delta \sC_{ \ell, \sH }\paren*{ h, x } }^{ \tfrac{1}{s} }.
\end{equation*}
Then, for every $h \in \sH$,
\begin{align*}
& \sE_{ \ell_{0-1} }\paren*{ h } - \sE^*_{ \ell_{0-1} }\paren*{ \sH } + \sM_{ \ell_{0-1} }\paren*{ \sH }\\
&\quad \leq \E_{X}\bracket*{ 1_{ \mu(h, X) > 0 } }^{\frac{1}{t}}
\paren*{\sE_{\ell}(h) - \sE^{*}_{\ell}(\sH) + \sM_{\ell}(\sH) }^{\frac{1}{s}}.
\end{align*}
\end{restatable}

\begin{proof}
The total $0$-$1$ estimation error is the expectation of the conditional regret:
\begin{equation*}
\sE_{ \ell_{0-1} }\paren*{ h } - \sE^*_{ \ell_{0-1} }\paren*{ \sH } + \sM_{ \ell_{0-1} }\paren*{ \sH }
= \E_X\bracket*{ \Delta \sC_{ \ell_{0-1}, \sH }\paren*{ h, X } }.
\end{equation*}
From Lemma~\ref{lemma:explicit_assumption_01} and the definition of
$\mu(h, x)$, we have
$\Delta \sC_{ \ell_{0-1}, \sH }\paren*{ h, x } \leq \mu(h, x)$. This
implies that if $\mu(h, x) = 0$, then
$\Delta \sC_{ \ell_{0-1}, \sH }\paren*{ h, x } = 0$. We can therefore
write
\begin{align*}
\E_X\bracket*{ \Delta \sC_{ \ell_{0-1}, \sH }\paren*{ h, X } } 
& = \E_X\bracket*{ \Delta \sC_{ \ell_{0-1}, \sH }\paren*{ h, X }
  \cdot 1_{ \mu(h, X) > 0 } } \\
& \leq \E_X\bracket*{ \paren*{ \Delta \sC_{ \ell, \sH }
  \paren*{ h, X } }^{ \tfrac{1}{s} } \cdot 1_{ \mu(h, X) > 0 } },
\end{align*}
where the inequality uses the theorem's pointwise assumption. We now apply Hölder's inequality with conjugate exponents $s \geq 1$ and $t \geq 1$ such that $\frac{1}{s} + \frac{1}{t} = 1$:
\begin{align*}
& \E_X\bracket*{ \paren*{ \Delta \sC_{ \ell, \sH }\paren*{ h, X } }^{ \tfrac{1}{s} } \cdot 1_{ \mu(h, X) > 0 } } \\
& \leq \paren*{ \E_X\bracket*{ \paren*{ \paren*{ \Delta \sC_{ \ell, \sH }\paren*{ h, X } }^{ \tfrac{1}{s} } }^s } }^{\frac{1}{s}} 
\paren*{ \E_X\bracket*{ \paren*{ 1_{\mu(h, X) > 0} }^t } }^{\frac{1}{t}} \\
& = \paren*{ \E_X\bracket*{ \Delta \sC_{ \ell, \sH }\paren*{ h, X } } }^{\frac{1}{s}} \cdot \paren*{ \E_X\bracket*{ 1_{\mu(h, X) > 0} } }^{\frac{1}{t}}
\end{align*}
since $1^t = 1$. Substituting the definition of the surrogate estimation error, $\E_X\bracket*{ \Delta \sC_{ \ell, \sH }\paren*{ h, X } } = \sE_{\ell}(h) - \sE^{*}_{\ell}(\sH) + \sM_{\ell}(\sH)$, yields the claimed bound.
\end{proof}
Theorem~\ref{Thm:mm-multi} provides a strong theoretical
guarantee. Specifically, it shows that if the functional form of
$\Gamma(x) = x^{\frac1s}$ is applied for standard $\sH$-consistency
bounds, then, the constant can be improved from $1$ to a more refined,
hypothesis-dependent quantity:
$\E_X[1_{\mu(h, X) > 0}]^{\frac{1}{t}} \leq 1$. This refinement
applies to all existing $\sH$-consistency bounds in
\citep{awasthi2022multi,mao2023cross,MaoMohriZhong2024}. In
particular, for many cases where the surrogate loss $\ell$ is smooth,
we have $t = s = \frac12$, as shown by
\citep{MaoMohriZhong2024}. Next, we assume that the MM noise assumption holds and that the approximation error $\sE^*_{ \ell_{0-1} }(\sH) = \sE_{\ell_{0-1}}(h^*) = \sE_{\ell_{0-1}}^*\paren*{\sH} - \sE_{\ell_{0-1}}^*\paren*{\sH_{\rm all}} = 0$. This also implies that the minimizability gap $\sM_{\ell_{0-1}}(\sH) = 0$. The following
result provides the corresponding $\sH$-consistency bound in this
case.

\begin{restatable}{theorem}{MMMultiLowNoise}
\label{Thm:mm-multi-low-noise}
Suppose $\sE^*_{ \ell_{0-1} }(\sH) =
\sE_{\ell_{0-1}}(h^*)$ and there exists
$s \geq 1$ with conjugate number $t \geq 1$, that is
$\frac{1}{s} + \frac{1}{t} = 1$, such that
\begin{equation*}
  \forall (h, x) \in \sH \times \sX,
  \quad \Delta \sC_{ \ell_{0-1}, \sH }\paren*{ h, x }
  \leq \paren*{ \Delta \sC_{ \ell, \sH }\paren*{ h, x } }^{ \tfrac{1}{s} }.
\end{equation*}
Then, under MM noise, there exists $c > 0$ such that for all $h \in \sH$,
\begin{align*}
& \sE_{ \ell_{0-1} }(h) - \sE^{*}_{ \ell_{0-1} }(\sH) + \sM_{\ell_{0-1}}\paren*{ \sH }\\
& \quad \leq c^{ \frac{ s - 1 }{ s - \alpha\paren*{ s - 1 } } }
  \bracket*{ \sE_{\ell}\paren*{ h } - \sE^*_{\ell}\paren*{ \sH }
  + \sM_{\ell}\paren*{ \sH } }^{ \tfrac{ 1 }{ s - \alpha\paren*{ s - 1 } } }.
\end{align*}
\end{restatable}

\begin{proof}
We start from Theorem~\ref{Thm:mm-multi}. Let $A = \sE_{ \ell_{0-1} }(h) - \sE^{*}_{ \ell_{0-1} }(\sH)$
and $B = \sE_{\ell}(h) - \sE^{*}_{\ell}(\sH) + \sM_{\ell}(\sH)$. By
Lemma~\ref{lemma:MM} and the assumption $\sE^*_{ \ell_{0-1} }(\sH) =
\sE_{\ell_{0-1}}(h^*)$, there exists $c > 0$ with $\E_{X}\bracket*{
1_{ \mu(h, X) > 0 } } \leq c A^{\alpha}$.
Using $\sM_{ \ell_{0-1} }\paren*{ \sH } = 0$, Theorem~\ref{Thm:mm-multi}
gives $A \leq \paren*{c A^\alpha}^{\frac{1}{t}} B^{\frac{1}{s}} =
c^{\frac{1}{t}} A^{\frac{\alpha}{t}} B^{\frac{1}{s}}$. Rearranging
yields $A^{ 1 - \alpha/t } \leq c^{\frac{1}{t}} B^{\frac{1}{s}}$.
Using $\frac{1}{t} = \frac{s-1}{s}$, the exponent is $1 - \alpha/t =
\frac{s - \alpha(s-1)}{s}$. Raising both sides to the power of
$\frac{1}{1 - \alpha/t} = \frac{s}{s - \alpha(s-1)}$ gives:
\begin{align*}
A & \leq \paren*{ c^{\frac{s - 1}{s}} B^{\frac{1}{s}} }^{ \frac{s}{ s - \alpha (s - 1) } }
= c^{ \frac{ s - 1 }{ s - \alpha (s - 1) } } B^{ \frac{ 1 }{ s - \alpha (s - 1) } },
\end{align*}
which is the claimed bound.
\end{proof}

\ignore{
\begin{proof}
By Lemma~\ref{lemma:MM}, there exists $c = \tfrac{ B^{ 1-\alpha } }{ \alpha^{\alpha} } > 0$ such that
\begin{equation*}
\E_{X}\bracket*{ 1_{ \mu(h, X) > 0 } }
\leq c \paren*{ \sE_{ \ell_{0-1} }(h) - \sE^{*}_{ \ell_{0-1} }(\sH) }^{\alpha}.    
\end{equation*}
Combining this bound with the bound in Theorem~\ref{Thm:mm-multi} yields
\begin{align*}
A &\leq c^{\frac{1}{t}} A^{\frac{\alpha}{t}} B^{\frac{1}{s}}, \text{where }\\
A &= \sE_{ \ell_{0-1} }(h) - \sE^{*}_{ \ell_{0-1} }(\sH),\\
B &= \sE_{\ell}(h) - \sE^{*}_{\ell}(\sH) + \sM_{\ell}(\sH).    
\end{align*}
Rearranging,
\begin{equation*}
A^{ 1 - \frac{\alpha}{t} } \leq c^{\frac{1}{t}} B^{\frac{1}{s}}.
\end{equation*}
Since $\frac{1}{t} = 1 - \frac{1}{s} = \frac{s - 1}{s}$, we have
\begin{equation*}
1 - \frac{\alpha}{t} = 1 - \alpha \frac{s - 1}{s} = \frac{ s - \alpha (s - 1) }{ s }.
\end{equation*}
Hence
\begin{equation*}
A \leq c^{ \frac{1}{ t \paren*{ 1 - \frac{\alpha}{t} } } } B^{ \frac{1}{ s \paren*{ 1 - \frac{\alpha}{t} } } }
= c^{ \frac{ s - 1 }{ s - \alpha (s - 1) } }
B^{ \frac{ 1 }{ s - \alpha (s - 1) } },
\end{equation*}
which is the claimed bound.
\end{proof}
}
The exponent $\tfrac{1}{ s - \alpha\paren*{ s - 1 } }$ matches that
derived from Tsybakov noise condition
\citep[Theorem~9]{mao2025enhanced}, but our bounds are established
under the weaker MM noise assumption. For smooth binary surrogates
($s = 2$), this result demonstrates the interpolation between a
square-root rate (as $\alpha \to 0$) and a linear rate (as
$\alpha \to 1$).

\subsection{Binary classification}
\label{sec:binary}

The binary analogue, which can be viewed as a special case of the multi-class setting, satisfies the same bound via an identical proof.

\begin{restatable}{theorem}{MMBinary}
\label{Thm:mm-binary}
Suppose there exists $s \geq 1$  with
conjugate number $t \geq 1$, that is $\frac{1}{s} + \frac{1}{t} = 1$, such that
\begin{equation*}
\forall (h, x) \in \sH \times \sX, \quad  \Delta \sC_{ \ell^{\rm bi}_{0-1}, \sH }\paren*{ h, x } \leq \paren*{ \Delta \sC_{ \ell, \sH }\paren*{ h, x } }^{ \tfrac{1}{s} }.
\end{equation*}
Then, for every $h \in \sH$,
\begin{align*}
& \sE_{ \ell^{\rm bi}_{0-1} }\paren*{ h } - \sE^*_{ \ell^{\rm bi}_{0-1} }\paren*{ \sH } + \sM_{ \ell^{\rm bi}_{0-1} }\paren*{ \sH }\\
&\quad \leq \E_{X}\bracket*{ 1_{ \mu(h, X) > 0 } }^{\frac{1}{t}}
\paren*{\sE_{\ell}(h) - \sE^{*}_{\ell}(\sH) + \sM_{\ell}(\sH) }^{\frac{1}{s}}.
\end{align*}
Under MM noise and $\sM_{ \ell^{\rm bi}_{0-1} }\paren*{ \sH } = 0$, there exists $c > 0$, 
\begin{align*}
\forall h \in \sH,\; &\sE_{ \ell^{\rm bi}_{0-1} }\paren*{ h } - \sE^*_{ \ell^{\rm bi}_{0-1} }\paren*{ \sH }\\
& \leq c^{ \frac{ s - 1 }{ s - \alpha\paren*{ s - 1 } } }
\paren*{ \sE_{\ell}\paren*{ h } - \sE^*_{\ell}\paren*{ \sH } + \sM_{\ell}\paren*{ \sH } }^{ \tfrac{ 1 }{ s - \alpha\paren*{ s - 1 } } }.
\end{align*}
\end{restatable}
\begin{proof}
The proofs are identical to that of Theorem~\ref{Thm:mm-multi} and Theorem~\ref{Thm:mm-multi-low-noise}, replacing the multi-class zero-one loss $\ell_{0-1}$ with the binary zero-one loss $\ell_{0-1}^{\rm bi}$.
\end{proof}

These results are analogous to the enhanced $\sH$-consistency bounds established under Tsybakov noise by \citet{mao2025enhanced}, but with the key distinction that the MM condition is \emph{weaker} and hypothesis-dependent. This shift also yields \emph{predictor-dependent} constants, quantified by $\E \bracket*{1_{\mu(h, X) > 0}}^{1 / t}$, which are absent from purely distributional Tsybakov analyses and align better with model selection. Crucially, we retain the desirable noise exponent $1/ \paren*{s-\alpha\paren*{s-1}}$ derived under Tsybakov noise. Consequently, practitioners gain bounds with the same favorable rates but under a less restrictive, model-aware assumption.

\section{Applications}
\label{sec:more}

\textbf{Structural properties.} The MM noise condition exhibits two key structural properties: \emph{(i) Monotonicity in $\sH$.} If $\sH_1 \subseteq \sH_2$ and the MM noise condition holds uniformly over $\sH_2$, then it also holds over $\sH_1$ with the same parameters.
\emph{(ii) Invariance.} If the scores $h \paren*{x, \cdot}$ are replaced by $\psi \circ h \paren*{x, \cdot}$ for any strictly increasing function $\psi$, the resulting predictions $\hh$ and the Bayes prediction $\hh^*$ remain unchanged. Consequently, the model margin $\mu\paren*{h, x}$ and the MM noise condition are also invariant.

These properties demonstrate the flexibility of the MM noise condition, highlighting its robustness to the specific choice of hypothesis set and the scaling of the prediction scores.

\textbf{Examples.} We now instantiate the derived bounds for common surrogate loss families. Table~\ref{tab:example-binary-mm} applies Theorem~\ref{Thm:mm-binary} to common binary margin-based losses in \citep{awasthi2022h}. The results confirm the linear rate ($s=1$) for the hinge loss and the $1/(2-\alpha)$ exponent for smooth surrogates ($s=2$), which interpolates between the square-root and linear regimes based on the noise parameter $\alpha$.

\begin{table}[h]
\centering
\resizebox{\columnwidth}{!}{
\begin{tabular}{l|l|l|l}
Loss & $\Phi(z)$ & $s$ & Bound under MM noise \\
\hline
Hinge
&
$\bracket*{ 1 - z }_{+}$
&
$1$
&
$\sE_{\ell}(h) - \sE^*_{\ell}(\sH)$
\\
Logistic
&
$\log\paren*{ 1 + e^{-z} }$
&
$2$
&
$c^{\tfrac{1}{ 2 - \alpha }} \bracket*{ \sE_{\ell}(h) - \sE^*_{\ell}(\sH) }^{ \tfrac{1}{ 2 - \alpha } }$
\\
Exponential
&
$e^{-z}$
&
$2$
&
$c^{\tfrac{1}{ 2 - \alpha }} \bracket*{ \sE_{\ell}(h) - \sE^*_{\ell}(\sH) }^{ \tfrac{1}{ 2 - \alpha } }$
\\
Sq-hinge
&
$\bracket*{ 1 - z }_{+}^{2}$
&
$2$
&
$c^{\tfrac{1}{ 2 - \alpha }} \bracket*{ \sE_{\ell}(h) - \sE^*_{\ell}(\sH) }^{ \tfrac{1}{ 2 - \alpha } }$
\\
\end{tabular}
}
\caption{Binary margin-based losses of the form
  $\ell_{\Phi}(h,x,y) = \Phi\paren*{ y\,h(x) }$ with corresponding
  $\sH$-consistency bounds.}
\label{tab:example-binary-mm}
\vskip -0.1in
\end{table}
Table~\ref{tab:example-mc-comp-mm} illustrates the bounds for comp-sum
losses in \citep{mao2023cross}, a broad family that generalizes
sum-type and cross-entropy-like surrogates. Under MM noise, the same
exponents as in the enhanced $\sH$-consistency bounds of
\citet{mao2025enhanced} are recovered, but they hold under the weaker
MM noise assumptions.

\begin{table}[h]
\centering
\resizebox{\columnwidth}{!}{
\begin{tabular}{l|l|l|l}
Loss & $\ell(h,x,y)$ & $s$ & Bound under MM noise \\
\hline
Mean absolute error (MAE)
&
$1 - p_{\theta}\paren*{ y \mid x }$
&
$1$
&
$\sE_{\ell}(h) - \sE^*_{\ell}(\sH)$
\\
Cross-entropy (Logistic)
&
$-\log \paren*{p_{\theta}\paren*{ y \mid x }}$
&
$2$
&
$c^{\tfrac{1}{ 2 - \alpha }} \bracket*{ \sE_{\ell}(h) - \sE^*_{\ell}(\sH) }^{ \tfrac{1}{ 2 - \alpha } }$
\\
Exponential comp-sum
&
$\frac{1}{p_{\theta}\paren*{ y \mid x }} - 1$
&
$2$
&
$c^{\tfrac{1}{ 2 - \alpha }} \bracket*{ \sE_{\ell}(h) - \sE^*_{\ell}(\sH) }^{ \tfrac{1}{ 2 - \alpha } }$
\\
Generalized cross-entropy
&
$\dfrac{ 1 - p_{\theta}\paren*{ y \mid x }^{q - 1} }{ q - 1 }$
&
$2$
&
$c^{\tfrac{1}{ 2 - \alpha }} \bracket*{ \sE_{\ell}(h) - \sE^*_{\ell}(\sH) }^{ \tfrac{1}{ 2 - \alpha } }$
\\
\end{tabular}
}
\caption{ Multi-class comp-sum losses.  Here
  $p_{\theta}\paren*{ y \mid x } = \exp\paren*{ h(x,y) } \big/
  \sum_{y'} \exp\paren*{ h(x,y') }$ is the softmax model.}
\label{tab:example-mc-comp-mm}
\end{table}

\section{Conclusion}

We introduced the MM noise condition, a \emph{weaker yet more
  flexible} hypothesis-dependent noise model than Tsybakov noise, and
derived enhanced $\sH$-consistency bounds for both binary and
multi-class classification. These results broaden the applicability of
existing bounds and provide adaptive guarantees under hypothesis-dependent low-noise conditions. Future work includes extending MM noise analysis
to regression settings and developing adaptive algorithms that
leverage its flexibility in practice.

\bibliography{mmhcb}

\end{document}